\documentclass{article}

     \usepackage[preprint]{neurips_2019}

\usepackage[utf8]{inputenc} 
\usepackage[T1]{fontenc}    
\usepackage{hyperref}       
\usepackage{url}            
\usepackage{booktabs}       
\usepackage{amsfonts}       
\usepackage{nicefrac}       
\usepackage{microtype}      
\usepackage{amsmath,amsthm,amssymb}
\usepackage{graphicx}
\usepackage[linesnumbered,ruled,noend]{algorithm2e} 
\usepackage{xcolor}
\usepackage{multirow}
\usepackage{mathtools}     
\usepackage{subfigure}
\usepackage{caption}
\usepackage{textcomp}      

\usepackage{float}
\newcommand{\vect}[1]{\mathbf{#1}}

\newcommand{\budget}{B}
\newcommand*{\tran}{^{\mathsf{T}}}

\newcommand*{\lrate}{\gamma}

          \newtheorem{theorem}{Theorem}[section]

\newtheorem{lemma}[theorem]{Lemma}

\newtheorem{definition}[theorem]{Definition}

\title{Solving Large-Scale 0-1 Knapsack Problems and Its Application to Point Cloud Resampling}

\author{%
  Duanshun Li \\
  University of Alberta\\
  \texttt{duanshun@ualberta.ca} \\
   \And
   Jing Liu, Noseong Park \\
   George Mason University\\
   \texttt{\{jliu30, npark9\}@gmu.edu} \\
   \AND
   Dongeun Lee \\
   Texas A\&M University-Commerce \\
   \texttt{Dongeun.Lee@tamuc.edu} \\
   \And
   Giridhar Ramachandran, Ali Seyedmazloom \\
   George Mason University\\
   \texttt{\{gramacha,aseyedma\}@gmu.edu} \\
   \And
   Kookjin Lee \\
   Sandia National Laboratories\\
   \texttt{koolee@sandia.gov} \\
   \And
   Chen Feng \\
   New York University\\
   \texttt{cfeng@nyu.edu} \\
   \And
   Vadim Sokolov, Rajesh Ganesan \\
   George Mason University\\
   \texttt{\{vsokolov,rganesan\}@gmu.edu} \\
}

\begin{document}

\maketitle

\begin{abstract}
\let\thefootnote\relax\footnotetext{Duanshun Li, Jing Liu and Noseong Park are listed in alphabetical order and equally contributed.}\let\thefootnote\relax\footnotetext{Noseong Park is the corresponding author.}0-1 knapsack is of fundamental importance in computer science, business, operations research, etc. In this paper, we present a deep learning technique-based method to solve large-scale 0-1 knapsack problems where the number of products (items) is large and/or the values of products are not necessarily predetermined but decided by an external value assignment function during the optimization process. Our solution is greatly inspired by the method of Lagrange multiplier and some recent adoptions of game theory to deep learning. After formally defining our proposed method based on them, we develop an adaptive gradient ascent method to stabilize its optimization process. In our experiments, the presented method solves all the large-scale benchmark KP instances in about a minute, whereas existing methods show fluctuating runtime. We also show that our method can be used for other applications, including but not limited to the point cloud resampling.
\end{abstract}

\section{Introduction}

0-1 knapsack problems (KPs), which choose an optimal subset of products (items) that maximize an objective function under a budget constraint, are a classical NP-hard problem that frequently occurs in real world applications. According to a study, KPs are one of the top-20 most popular problems~\cite{KelPfePis04,Skiena:1999:IAW:333623.333627}. Its application areas include decision making process, resource allocation, cryptography, computer vision, and so forth, to name a few.

Many algorithms have been proposed to efficiently solve them~\cite{KelPfePis04,Martello:1990:KPA:98124}. Some of them are classical combinatorial optimization algorithms whereas others rely on deep learning methods~\cite{DBLP:journals/corr/BelloPLNB16}. Classical algorithms resort to integer linear programming, dynamic programming, and so on. However, their scalability is known to be unsatisfactory oftentimes. As the number of products $n$ increases, their runtime exponentially increases~\cite{Garey:1979:CIG:578533}.


Recently, a few deep learning-based methods to \emph{predict} solutions were proposed~\cite{NIPS2015_5866,DBLP:journals/corr/BelloPLNB16,DBLP:journals/corr/ZophL16}. Note that we use the term `predict' instead of `solve' because given a problem instance they do predict its solution after training with many problem instance-solution pairs or with specially designed reinforcement learning. Most of them aimed at solving the traveling salesman problem (TSP) --- few of them experimented with KPs. However, they also suffer from the same scalability issue.

KP instances with a large number of products, e.g., $n > 1,000,000$, have never been conquered by any methods in a stable manner --- in general, the runtime of existing methods varies a lot case by case and in some instances, they cannot solve in time. Moreover, product values are not constant sometimes, which makes many existing solutions unapplicable. In this paper, we propose a method to solve such large-scale KPs using advanced deep learning techniques --- we do not predict solutions.

However, one difficulty in designing such a constrained optimization technique on top of deep learning platforms is how to consider the \emph{hard} budget constraint --- we must not violate the budget limitation. We rely on the Karush–Kuhn–Tucker (KKT) conditions and the Lagrange multiplier method, a standard method to convert a constrained problem into a corresponding unconstrained problem. However, we do not stop at applying the Lagrange multiplier method in a naive way but propose a more efficient max-min game formulation (derived from the Lagrange multiplier method) and an adaptive gradient ascent to ensure the hard budget constraint. Our contributions are as follows:
\begin{enumerate}
    \item In deep learning platforms, it is not straightforward to consider the budget constraint of KPs. To resolve this, we design a max-min game (whose equilibrium state is equivalent to the optimal KP solution) and propose an adaptive gradient ascent method that is able to optimize the KP objective function with no cost overrun.
    \item We propose a simple but effective neural network to solve KPs. This proposed neural network is trained with the proposed adaptive gradient ascent method.
    \item We conduct extensive experiments with state-of-the-art baseline methods and benchmark KP datasets.
    \item We show that our proposed method can be used for other applications, including but not limited to point cloud resampling.
    \item Program codes, data, and Appendix will be released upon publication.
\end{enumerate}

\section{Related Work}
The 0-1 KP with $n$ products (items), where a product $i$ has value $v_i$ and cost $c_i$,  is to find binary variables that optimize the following problem:
\begin{align}\begin{split}\label{eq:orig}
    \max_{x_i \in \{0, 1\}} \quad & \sum_{i=1}^{n} v_i x_i,\\
    \textrm{subject to }\quad & \sum_{i=1}^{n} c_i x_i \leq \budget,
\end{split}\end{align}
where $\budget$ is the total budget; $x_i \in \{0,1\}$ denotes if we choose the product $i$ or not.

The value $v_i$ is a constant in many cases but in some more complicated KP instances, $v_i$ is not predetermined and the objective is defined by a value assignment function $g:\{0,1\}^n \rightarrow \mathbb{R}$ as follows:
\begin{align}\begin{split}\label{eq:func}
    \max_{x_i \in \{0, 1\}} \quad & g(x_1, \cdots, x_n),\\
    \textrm{subject to }\quad & \sum_{i=1}^{n} c_i x_i \leq \budget.
\end{split}\end{align}

Therefore, Eq.~\eqref{eq:orig} can be considered as a special case of Eq.~\eqref{eq:func}. In our study, we conduct experiments on both cases. In particular, $g$ is represented by a neural network in one of our experiments.

Many algorithms specialized to KPs~\cite{Pisinger93anexpanding-core,1997:MAK:2752768.2752778,doi:10.1287/mnsc.45.3.414} have been proposed so far. In addition, other general-purpose solvers have been utilized to solve KPs, such as Gurobi, LocalSolver, OR-Tools by Google AI, etc. However, most of these methods consider only Eq.~\eqref{eq:orig} where product values are fixed.

A few recent works \textit{predict} solutions of optimization problems using neural networks~\cite{NIPS2015_5866,DBLP:journals/corr/BelloPLNB16,DBLP:journals/corr/ZophL16}, which mostly focus on the traveling salesman problem (TSP). However, most of these works require numerous pairs of problem instance and its solution to train their prediction models. To produce such training samples, they need to run other combinatorial optimization algorithms, which is time-consuming --- recall that both TSP and KP are NP-hard.

Recently, one method was proposed that does not require the generation of such training samples by Bello et al~\cite{DBLP:journals/corr/BelloPLNB16}. In their work, a reinforcement learning algorithm is used to train the solution prediction model. Their reinforcement learning uses an actor-critic training method where the actor produces a solution and the critic returns its reward (objective value). Nevertheless, their algorithm should still be trained with many problem instances to stabilize the solution quality. And most importantly, its LSTM-based actor is not capable of processing a sequence of $n > 1,000,000$ products. They experimented with at most hundreds of products. However, KP instances with less than 10,000 products can be exactly solved by classical methods in a few seconds~\cite{Pisinger:2005:HKP:1063636.1063640}, which leads to little motivation to use the reinforcement learning-based method in real word environments.

\section{Proposed Method}
We first define our neural network that can be described by a function $f:\theta \rightarrow \{0,1\}^n$, where $\theta$ is the neural network weights. In other words, our neural network $f$ outputs a set of selected products (items) and does not require any inputs other than $\theta$. The input KP instance's values and costs are used to describe the training loss and we train $\theta$ using the loss with our novel gradient ascent method. As mentioned earlier, our neural network $f$ does not predict but solve the input KP instance.


The architecture of $f$ is shown in Figure~\ref{fig:archi}. For each product $i$, we have a one-dimensional parameter $e_i \in \mathbb{R}$ that will be trained --- in fact, these parameters are the only ones to learn in our model, i.e., $\theta=\{e_1,e_2,\ldots,e_n\}$. The following straight-through and pass-through estimators produce $x_i \in \{0,1\}$ from $e_i$, where $i=1,2,\ldots,n$. In the straight-through estimator (STE),
\begin{align*}
    x_i = \begin{cases} Round(Sigmoid(\tau e_i)),\textrm{ if forward-pass}\\
     Sigmoid(\tau e_i),\textrm{ if backward-pass}
    \end{cases},
\end{align*}where $\tau \geq 1$ is a slope annealing parameter and we perform $\tau \gets r ^ {(p / s)} \tau$ at epoch $p$ with the change rate $r > 1$ and the step parameter $s$~\cite{DBLP:journals/corr/ChungAB16}.

In the pass-through estimator (PTE),
\begin{align*}
    x_i = \begin{cases} Round(Sigmoid(e_i)),\textrm{ if forward-pass}\\
     e_i,\textrm{ if backward-pass}
    \end{cases}.
\end{align*}

These two estimators are popular techniques to train binary stochastic neurons~\cite{DBLP:journals/corr/BengioLC13}. If $\tau$ is large enough, the Sigmoid itself becomes close to the binary rounding function in STE.

The proposed neural network is trained to maximize the KP objective and after training, we take $x_i \in \{0,1\}$ to know which products have been selected. We will describe shortly how we consider the budget constraint.

\begin{figure}
    \centering
    \includegraphics[width=0.5\textwidth]{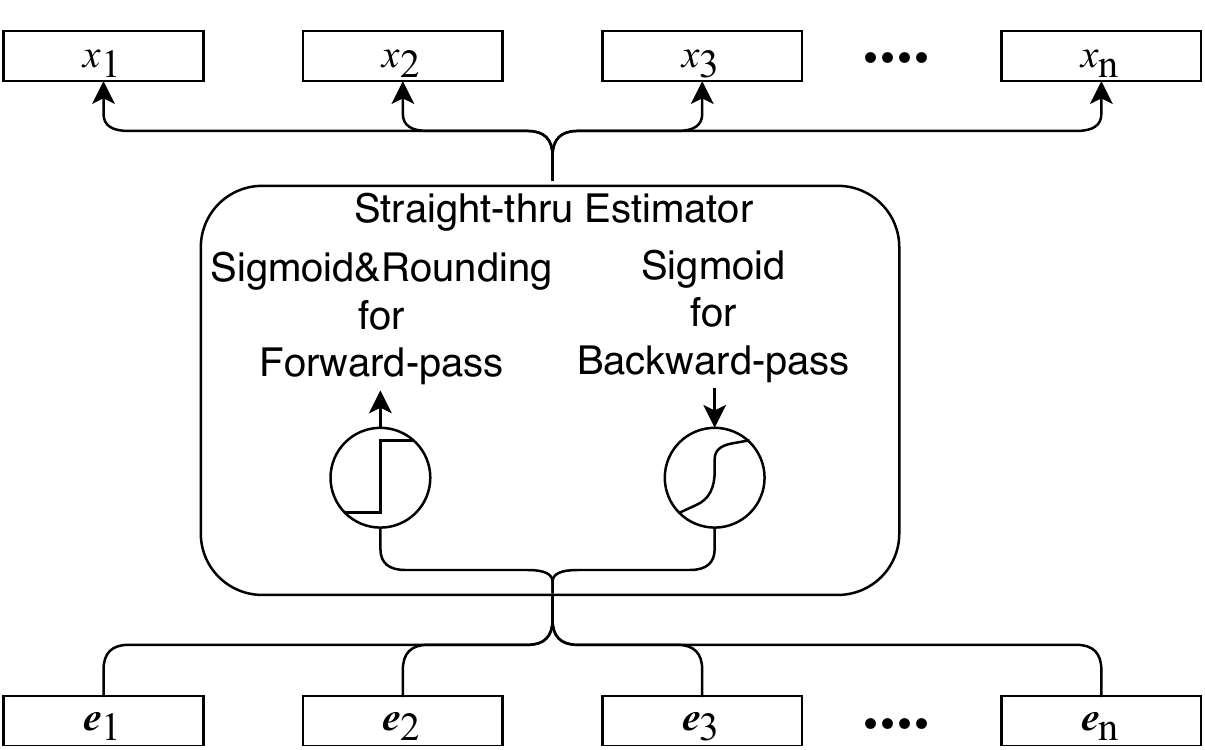}
    \caption{Proposed neural network generating $x_i \in \{0,1\}$ given $e_i \in \mathcal{R}$. In this figure we show the straight-through estimator (STE) which works as the rounding after the Sigmoid for its forward-pass and as the Sigmoid activation only for its backward-pass. We test with the pass-through estimator (PTE) as well. We do not have any intermediate linear, convolutional, or recurrent layers; rather, we directly feed $e_i$ into the estimator.}
    \label{fig:archi}
\end{figure}

With the proposed neural network, KP can be formulated as follows:
\begin{align}\begin{split}\label{eq:f1}
    \max_{\theta} \quad & \vect{v}\tran f(\vect{\theta}),\\
    \textrm{subject to }\quad & \vect{c}\tran f(\vect{\theta}) \leq \budget,
\end{split}\end{align}
or 
\begin{align}\begin{split}\label{eq:f2}
    \max_{\theta} \quad & g(f(\vect{\theta})),\\
    \textrm{subject to }\quad & \vect{c}\tran f(\vect{\theta}) \leq \budget,
\end{split}\end{align}where $\vect{\theta}$ represents the neural network parameters of $f$, $f(\theta) = \vect{x}= [ x_1, x_2, \ldots, x_n ]\tran$, $\vect{v} = [ v_1, v_2, \ldots, v_n ]\tran$, and $\vect{c} = [ c_1, c_2, \ldots, c_n ]\tran$. We use bold fonts to denote (column) vectors.

\noindent\textbf{Remark.}
The space complexity of $f$ is $\mathcal{O}(n)$ where $n$ is the number of products --- in other words, $n$ of one-dimensional parameters.

Because of this extremely space-efficient architecture, it only takes around a few hundred megabytes of GPU memory to solve with $n = 1,000,000$.


\subsection{Lagrange Multiplier Approach}
Without the budget constraint, maximizing $\vect{v}\tran \vect{x}$ or $g(\vect{x})$ is straightforward. To consider the constraint, we exploit the Karush–Kuhn–Tucker (KKT) conditions which tell us one necessary condition of optimality.

Under the KKT conditions, Eqs.~\eqref{eq:f1} and~\eqref{eq:f2} can be re-written as follows:
\begin{align}\label{eq:kkt}
\max_{\theta} \quad & o(f(\vect{\theta})) - \lambda b(f(\vect{\theta})),
\end{align} where $o(f(\vect{\theta}))=\vect{v}\tran f(\vect{\theta})=\vect{v}\tran \vect{x}$ (or $o(f(\vect{\theta}))=g(f(\vect{\theta}))=g(\vect{x})$), $b(f(\vect{\theta})) = \vect{c}\tran f(\vect{\theta}) - \budget = \vect{c}\tran \vect{x} - \budget$, and $\lambda \geq 0$ is the Lagrange multiplier.

In fact, the method of Lagrange multiplier had been already used to solve several deep learning problems. Our novelty lies in proposing a novel solving technique that we will discuss shortly.


\begin{definition}\label{def:kkt}
Given an objective function $o:\{0,1\}^n \rightarrow \mathbb{R}$ and a constraint function $b:\{0,1\}^n \rightarrow \mathbb{R}$, the KKT conditions say that there exists $\lambda$ such that $\nabla o(f(\vect{\theta}^*)) - \lambda \nabla b(f(\vect{\theta}^*)) = 0$ if $\vect{\theta}^*$ is a (local) optimum and the constraint function satisfies a regularity condition.
\end{definition}


Note that the above definition tells us the necessary condition of optimality. One of such regularity conditions is the linear independence of the gradients of constraints at $\vect{\theta}^*$. Since we have only one inequality constraint in our KP formulation, this regularity condition is met by the problem definition itself. 
Notice that this necessary condition does not mean that every $\theta$ with $\nabla o(f(\vect{\theta})) - \lambda \nabla b(f(\vect{\theta})) = 0$ is globally optimal. However, the Lagrange multiplier approach is widely used and works well in practice for finding a reasonable solution.

Unfortunately, solving Eq.~\eqref{eq:kkt} in practice can become unstable without a proper setting of $\lambda$. To this end, we propose the following max-min game to find a stable $\lambda$ and an optimized $\theta$:
\begin{align}\label{eq:kkt2}
\max_{\theta}\;\min_{\lambda \geq 0} \quad & o(f(\vect{\theta})) - \lambda b(f(\vect{\theta})) + \delta \lambda^2,
\end{align}where $\delta \lambda^2 \geq 0$ is a regularization term to prevent $\lambda$ from getting too large. If $b(f(\vect{\theta})) \leq 0$, i.e., no cost overrun, we prefer a small $\lambda$ to emphasize the KP  objective and the inner minimization makes it small. If $b(f(\vect{\theta})) > 0$, however, $\lambda$ should be large enough to emphasize the constraint while training $\theta$. Sometimes the inner minimization makes $\lambda$ larger than necessary (because $b(f(\vect{\theta})) > 0$ and a large $\lambda$ can minimize Eq.~\eqref{eq:kkt2} easily). To prevent this, we add the regularization term.


The rationale behind the proposed max-min game between $\theta$ and $\lambda$ is that $\theta$ tries to maximize the objective while $\lambda$ tries to suppress the selection of products to minimize the budget constraint term. However, by adding the regularization we try to limit $\lambda$ within a reasonable range. Therefore, it will converge to a balancing point between the objective and the budget constraint.

\begin{lemma}
Given a fixed $\theta$, $\lambda = \max\left\{0,\frac{b(f(\theta))}{2\delta}\right\}$ minimizes Eq.~\eqref{eq:kkt2}.
\end{lemma}
\begin{proof}
All proofs are in Appendix that will be released upon publication.
\end{proof}


Since we know the optimal value of $\lambda$, Eq.~\eqref{eq:kkt2} can be rewritten as follows:
\begin{align*}
\max_{\theta}\quad & o(f(\vect{\theta})) - \frac{\max\{0,b(f(\vect{\theta}))\}^2}{4\delta}.
\end{align*}
Equivalently,
\begin{align}\label{eq:kkt3}
\max_{\theta}\quad & o(f(\vect{\theta})) - \frac{1}{2}\beta R(b(f(\vect{\theta})))^2,
\end{align}where $\beta = \frac{1}{2\delta}$, and the function $R: \mathbb{R} \rightarrow \mathbb{R}^{+}$ denotes the rectifier.

We solve Eq.~\eqref{eq:kkt3} instead of Eq.~\eqref{eq:kkt2} because we can remove the inner minimization. Eq.~\eqref{eq:kkt3} has the squared budget constraint penalty, which greatly drives the optimization process toward the zero budget constraint point --- note that in KPs, optimal solutions are achieved when $b(f(\vect{\theta})) \lesssim 0$.

\begin{theorem}
The equilibrium state of the proposed max-min game corresponds to the optimal KP solution with a proper setting of $\beta$ mentioned in the proof.
\end{theorem}

The above equilibrium theorem says that the role of $\beta$ is important to make the optimization process converge to the optimal KP solution and a large $\beta$ is preferred. If $\beta$ is too large, however, it will disturb the gradient ascent algorithm as soon as there is a small cost overrun. Therefore, we propose the following adaptive gradient ascent algorithm that dynamically controls $\beta$.

\subsection{Adaptive Gradient Ascent Method}


Because we maximize $L = o(f(\vect{\theta})) - \frac{1}{2}\beta R(b(f(\vect{\theta})))^2$, we need to use the gradient ascent method with $\vect{\theta} \gets \vect{\theta}+\lrate \nabla L$ where $\lrate > 0$ is the learning rate. When $b(f(\vect{\theta}))\leq0$, i.e., no cost overrun, we do not need to control $\beta$ because of the rectifier $R$. However, we simply set $\beta=0$ to enable more aggressive search for the KP objective without needing to consider the budget constraint.

When $b(f(\vect{\theta}))>0$, however, we need to enforce that its steepest gradient ascent direction reduces the total cost. For ease of notation in this section, we simply use $b(\theta)$ and $o(\theta)$ to represent $b(f(\vect{\theta}))$ and $o(f(\vect{\theta}))$, and use $\vect{b}'$ and $\vect{o}'$ to represent their gradients w.r.t. $\vect{\theta}$. Using these notations, the steepest gradient ascent direction of $L$ w.r.t. $\theta$, when $b(\theta)>0$, will be as follows:
\begin{align*}
    \nabla L = \vect{o}' - \beta b(\theta)\vect{b}'.
\end{align*}

The following theorem shows the condition of $\beta$ that decreases the cost after one gradient ascent update.

\begin{theorem}\label{th:beta_selection_1}
For any $\vect{\theta}$ with $b(\vect{\theta})>0$, $b(\vect{\theta}+\lrate \nabla L)<b(\vect{\theta})$ with a sufficiently small $\lrate >0$ if and only if $\beta > \frac{\vect{b}' \cdot \vect{o}'}{b(\vect{\theta})\vect{b}' \cdot \vect{b}'}$.
\end{theorem}

At the same time, it is desirable that we do not decrease the KP objective after one gradient ascent update, i.e., $o(\vect{\theta}+\lrate \nabla L) \geq o(\vect{\theta})$. Note that it leads to the best outcome if we can achieve both $b(\vect{\theta}+\lrate \nabla L)<b(\vect{\theta})$ and $o(\vect{\theta}+\lrate \nabla L) \geq o(\vect{\theta})$.

\begin{theorem}\label{th:beta_selection_2}
For any $\vect{\theta}$ with $b(\vect{\theta})>0$, $o(\vect{\theta}+\lrate \nabla L) \geq o(\vect{\theta})$ with a sufficiently small $\lrate >0$ if and only if 
$\beta \begin{cases}
\geq \frac{\vect{o}' \cdot \vect{o}'}{b(\vect{\theta})\vect{b}' \cdot \vect{o}'},\textrm { if } \vect{b}' \cdot \vect{o}' < 0\\
\leq \frac{\vect{o}' \cdot \vect{o}'}{b(\vect{\theta})\vect{b}' \cdot \vect{o}'},\textrm { if } \vect{b}' \cdot \vect{o}' > 0
\end{cases}
$.
\end{theorem}

Note that in Thm.~\eqref{th:beta_selection_2}, $\beta$ need not be defined if $\vect{b}' \cdot \vect{o}'=0$, where $o(\vect{\theta}+\lrate \nabla L) \geq o(\vect{\theta})$ is automatically satisfied and we only consider Thm.~\eqref{th:beta_selection_1}. From the above two Thms.~\eqref{th:beta_selection_1} and~\eqref{th:beta_selection_2}, therefore, we can have the following conditions:
\begin{align*}
\left\{\begin{array}{ll} \beta \geq 0 &,\textrm{ if }\vect{b}' \cdot \vect{o}' < 0\\
\beta > 0 &,\textrm{ if }\vect{b}' \cdot \vect{o}' = 0\\
\frac{\vect{b}' \cdot \vect{o}'}{b(\vect{\theta})\vect{b}' \cdot \vect{b}'}<\beta \leq\frac{\vect{o}' \cdot \vect{o}'}{b(\vect{\theta})\vect{b}' \cdot \vect{o}'} &,\textrm{ if }\vect{b}' \cdot \vect{o}' > 0
\end{array}\right. .
\end{align*}

If $\vect{b}' \cdot \vect{o}' < 0$ in the above condition, then $\beta \geq 0$ satisfies both Thm.~\eqref{th:beta_selection_1} and Thm.~\eqref{th:beta_selection_2}. However, if it is not the case that $\frac{\vect{b}' \cdot \vect{o}'}{b(\vect{\theta})\vect{b}' \cdot \vect{b}'} < \frac{\vect{o}' \cdot \vect{o}'}{b(\vect{\theta})\vect{b}' \cdot \vect{o}'}$, then the decrease in the cost (Thm.~\eqref{th:beta_selection_1}) and the non-decrease in the objective (Thm.~\eqref{th:beta_selection_2}) cannot be fulfilled at the same time, in which case we rely only on Thm.~\eqref{th:beta_selection_1}. Considering these conditions, we use the following definition of $\beta$:
\begin{align}\label{eq:final}
\beta = \left\{\begin{array}{ll} 0 &,\textrm{ if }\vect{b}' \cdot \vect{o}' < 0\\
\frac{\vect{b}' \cdot \vect{o}'}{2b(\vect{\theta})\vect{b}' \cdot \vect{b}'} + \frac{\vect{o}' \cdot \vect{o}'}{2b(\vect{\theta})\vect{b}' \cdot \vect{o}'} &,\textrm{ if } (\vect{b}' \cdot \vect{o}' > 0) \wedge \left(\frac{\vect{b}' \cdot \vect{o}'}{b(\vect{\theta})\vect{b}' \cdot \vect{b}'} <  \frac{\vect{o}' \cdot \vect{o}'}{b(\vect{\theta})\vect{b}' \cdot \vect{o}'}\right)\\
\frac{\vect{b}' \cdot \vect{o}'}{b(\vect{\theta})\vect{b}' \cdot \vect{b}'} + \epsilon &,\textrm{ otherwise }
\end{array}\right.
\end{align}where $\epsilon >0$ is a small positive value that we can analytically decide by the following theorem.


\begin{theorem}
There will be no cost overrun after one gradient ascent update of $\theta$ if $\epsilon = \frac{1}{\lrate \vect{b}' \cdot \vect{b}'}$ and $b(\theta)$ is linear.
\end{theorem}

The final proposed adaptive gradient ascent method is shown in Alg.~\ref{alg:adaptive-gd}. Let $\theta=\{e_1,e_2,\ldots,e_n\}$ be the set of the one-dimensional parameters to learn. At line~\ref{alg:mini}, we choose a $\xi$ portion of $\theta$ as a mini-batch $\bar{\theta}$ which is updated at line~\ref{alg:opt} --- note that line~\ref{alg:opt} does not use any advanced optimizers, such as Adam, RMSprop, and so forth: we leave supporting other optimizers as future work. Line~\ref{alg:beta} is to dynamically adjust $\beta$ if there is cost overrun. Line~\ref{alg:stop} is to check the early-stopping condition. If there are no improvements in the past $k$ epochs, we stop the process.
\begin{figure}[t]
\begin{minipage}{0.4\textwidth}
\centering
\includegraphics[width=0.97\columnwidth,trim={0.3cm 0.3cm 0.7cm 0.3cm},clip]{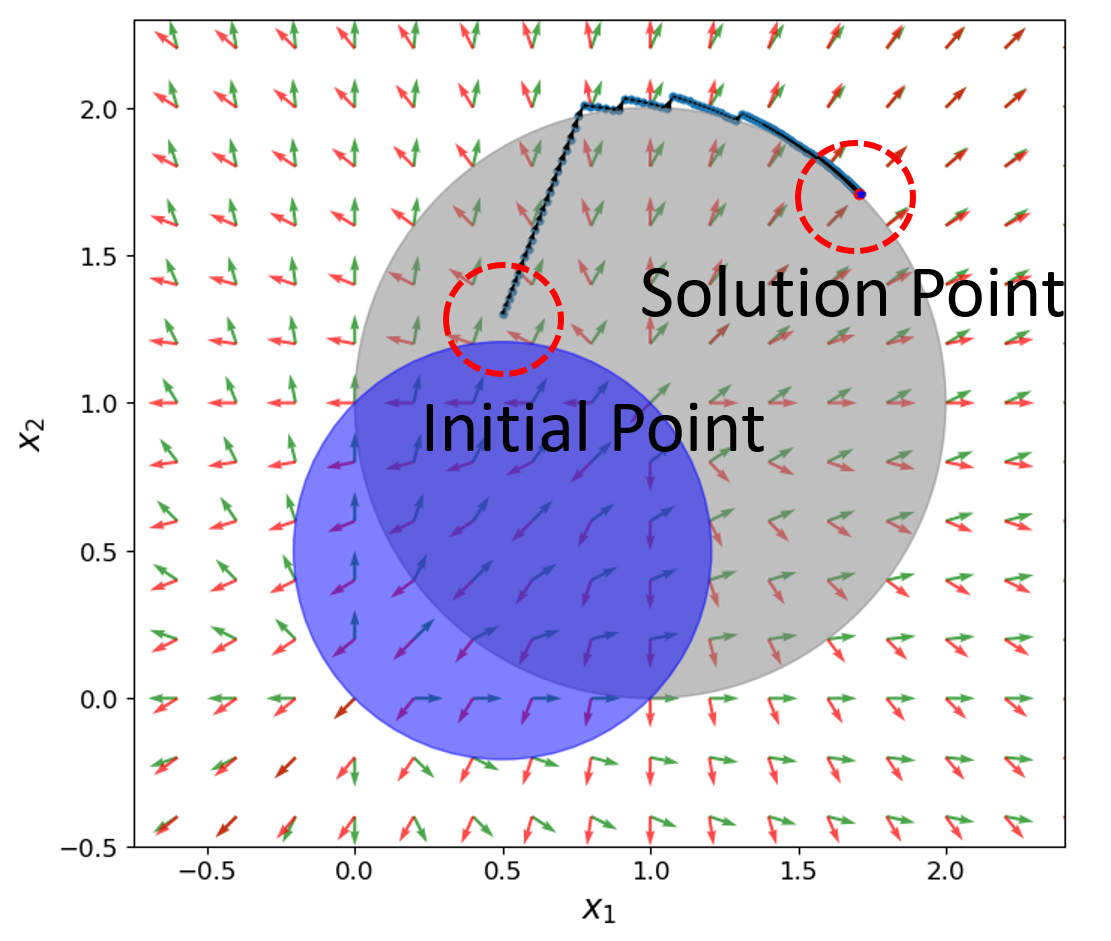}



\end{minipage}\hfill
\begin{minipage}{0.57\textwidth}
    	\begin{algorithm}[H]
    	\scriptsize
		\SetAlgoLined
		\caption{Adaptive gradient ascent}\label{alg:adaptive-gd}
		\KwIn{$\vect{\theta}_0$, $k$, $\lrate$}
		\KwOut{$\vect{x}$}
		$tolerance \gets k$\tcc*[r]{Initialize the early stopping threshold}
		$\vect{\theta} \gets \vect{\theta}_0$\tcc*[r]{Initialize neural network weights}
        $\beta \gets 0 $\tcc*[r]{Initialize $\beta$}
        $best \gets 0$\;
		\While {not converged}{
		    $\bar{\theta} \gets$ a $\xi$ portion of $\theta$\tcc*[r]{Mini-batch selection}\label{alg:mini}
			$\bar{\theta} \gets \bar{\theta}+\lrate \nabla L$\tcc*[r]{Gradient ascent}\label{alg:opt}
			$\vect{x} \gets f(\theta)$\;
			\eIf{$b(\vect{x}) > 0$ or $o(\vect{x}) \leq best$}{
			    $tolerance \gets tolerance - 1$\;
			}{
			    $tolerance \gets k$\;	
			    $best \gets o(\vect{x})$\;
			}
			\lIf{$tolerance \leq 0$}{break}\label{alg:stop}
			\eIf{$b(\vect{x})>0$}{
				$\beta \gets$ Eq.~\eqref{eq:final}\tcc*[r]{Dynamically adjust $\beta$}\label{alg:beta}
			}{
			$\beta \gets 0$ \;
			}
		}
\Return $\vect{x}$
	\end{algorithm}
\end{minipage}
\caption{Trajectory of optimizing $x_1^2 + x_2^2$ with a constraint $(x_1-1)^2 + (x_2-1)^2 \leq 1$ from an initialization point. Note that it converges to the optimal solution. Gray circle means there is no cost overrun and blue circle means $\mathbf{b}'\cdot\mathbf{o}' < 0$, i.e., their directions are opposite to each other. Green arrow represents $\mathbf{o}'$ and red arrow $\mathbf{b}'$. Note that our algorithm pulls the trajectory toward the gray circle as soon as there is any small cost overrun.}
	    \label{fig:eg}
\end{figure}

In Fig.~\ref{fig:eg}, we show an example of optimizing a simple mathematical function that consists of two variables $x_1$ and $x_2$ with a constraint. Starting from an initialization of $x_1$ and $x_2$, it successfully converges to the optimal solution. If there happens any small cost overrun, our method effectively removes it by controlling $\beta$ as shown in the figure.

\section{Experiments}
We test our proposed method with various 0-1 KPs. In the first experiment, we test with standard benchmark KP instances and in the second round, we show how our method can be used in the point cloud resampling. We used CUDA 10, TensorFlow 1.11.0, Numpy 1.14.5, Scipy 1.1.0, and machines with i9 CPU and GTX1080Ti. 

\subsection{Benchmark KP Instances}
We test with classical KP instances, whose problem formulations follow Eq.~\eqref{eq:f1}, in this subsection.

\noindent\textbf{Datasets.}
We use the benchmark KP instances released by Pisinger~\cite{Pisinger:2005:HKP:1063636.1063640}. It contains many realistic KP instances where product values/costs and budgets are already decided --- we do not change anything but strictly follow their instances; so we set $B$ to the budget given in each instance. Among them, we use one large-scale configuration, $n=1,000,000$. We are not interested in smaller scales because they can be exactly solved with many other existing techniques.

These KP instances can be classified into 6 types depending on the characteristics of $v_i$ and $c_i$ for each product $i$. One representative type is \emph{spanner} where $v_i$ and $c_i$ are multiples of a quite small set of numbers that is called \emph{spanner set}. Table~\ref{tbl:kp} shows the spanner set $span(2,10)$ where the two parameters 2 and 10 were chosen by Pisinger. We strictly follow his benchmark design.

\noindent\textbf{Baselines and Hyper-parameters.}
We tested seven baselines: Expknap~\cite{Pisinger93anexpanding-core}, Minknap~\cite{1997:MAK:2752768.2752778}, Combo~\cite{doi:10.1287/mnsc.45.3.414}, Gurobi~\cite{gurobi}, LocalSolver (LS)~\cite{ls}, OR-Tools by Google AI~\cite{ortools}, and a greedy algorithm. The greedy algorithm first sorts all products in descending order of $v_i/c_i$ and then makes greedy selections from the top ranked product until the budget runs out. It is already known that this greedy selection is close to optimal. However, this cannot be used to solve Eq.~\eqref{eq:f2} because product values are not fixed.

Except the greedy, all return optimal solutions. These baselines resort to many techniques such as dynamic programming, branch and bound, integer linear programming with LP-relaxation, Horowitz-Sahni decomposition, and so on. However, only few of them are scalable up to $n=1,000,000$. The reinforcement learning-based combinatorial optimization method~\cite{DBLP:journals/corr/BelloPLNB16} is not scalable because its LSTM-based actor cannot process a list of one million products. We give an hour for each solver to solve each instance. For our method, we use $\xi=0.1$ as the mini-batch size. We set the initialization of $e_i$ to zero for all $i$, $\lrate$ to 0.1, and $k$ to 100. We use STE with the slope annealing step size $s=50$ and the change rate $r=1.01$ to generate $x_i$ from $e_i$.


\noindent\textbf{Experimental Results.}
In Table~\ref{tbl:kp}, we summarize all the results. In general, Combo and Gurobi show better performance than other baselines but in a few types, they could not solve in an hour. OR-Tools, LS, Expknap, and Minknap are all timed out. We think that Gurobi shows great performance considering its wide applicability --- in fact, Gurobi is one of the most popular tools to solve (mixed integer) liner programming, quadratic programming, etc. Our proposed method solves all instances in less than about a minute and our method is the only one which solves all the instances.


\begin{table}[t]
\footnotesize
\centering
\caption{Results on benchmark KP instances with 6 types. Each type defines how to generate product values and costs. Due to the very large scale of objective values, we only show whether solution is optimal or not. `Timeout' means it cannot be solved in an hour in our experiments. OR-Tools, LS, Expknap, and Minknap are timed out in all cases and are not listed in this table.}\label{tbl:kp}
\begin{tabular}{|c|c|c|c|c|c|c|}
\hline
\multirow{2}{*}{Method} & \multicolumn{2}{c|}{Uncorr. Span(2,10)} & \multicolumn{2}{c|}{Weak. Corr. Span(2,10)} & \multicolumn{2}{c|}{Str. Corr. Span(2,10)} \\ \cline{2-7} 
 & Is optimal? & Time & Is optimal? & Time & Is optimal? & Time \\ \hline
Combo & O & 329 sec & O & 3018 sec & O & 19.13 sec\\ \hline
Greedy & X & < 1 sec & X & < 1 sec & X & < 1 sec \\ \hline
Gurobi & O & \textbf{2.61} sec & O & 2.78 sec & O & \textbf{2.70} sec  \\ \hline
Ours & O & 4.46 sec & O & \textbf{1.70} sec & O & 3.17 sec  \\ \hline\hline
\multirow{2}{*}{Method} & \multicolumn{2}{c|}{Strongly correlated} & \multicolumn{2}{c|}{Inverse strongly correlated} & \multicolumn{2}{c|}{Uncorrelated} \\ \cline{2-7} 
 & Is optimal? & Time & Is optimal? & Time & Is optimal? & Time  \\ \hline
Combo & O & \textbf{< 1} sec & O & \textbf{< 1} sec & N/A & Timeout  \\ \hline
Greedy & X & < 1 sec & X & < 1 sec & X & < 1 sec  \\ \hline
Gurobi & N/A & Timeout & O & 19.39sec & O & 21.71 sec \\ \hline
Ours & O & 60.16 sec & O & 4.72 sec & O & \textbf{< 1} sec   \\ \hline
\end{tabular}
\end{table}

\subsection{Point Cloud Resampling}
In this experiment, we show that our method works well to optimize a neural network-based objective. Point clouds are usually generated by a large set of collected (scanned) points on the external surface of objects from 3D scanners. Point cloud resampling (PCR) is to choose a subset of scanned points to reduce the space and time overheads of point processing algorithms. However, there are no well-defined metrics to evaluate and perform resampling. Each resampling algorithm relies on its own heuristic method.

We solve the KP problem in Eq.~\eqref{eq:f2} where $g$ includes a point cloud classification model and $B$ is the number of desired points in ratio --- i.e., find a $B$ portion of points that maximize the quality of resampling. We introduce our formal problem definition to solve as follows:
\begin{align}\begin{split}\label{eq:proposed}
    \min_{\theta} &\;\; \sigma(\ell, m(\theta_m, \mathbf{P} \circ f(\theta)))\\
    \textrm{subject to} &\;\; \frac{\|f(\theta)\|_1}{n} \leq B,
\end{split}\end{align}where $\mathbf{P} \in \mathbb{R}^{n \times 3}$ is a set (matrix) of 3-dimensional points scanned from one object; $m$ is a point cloud classification model, in particular PointNet~\cite{DBLP:journals/corr/QiSMG16} in our experiments; $\theta_m$ is its model parameters; $f(\theta)$ outputs a column vector $\mathbf{x} \in \{0,1\}^{n}$; `$\circ$' denotes the row-wise multiplication\footnote{Because of the PointNet design, this row-wise multiplication is identical to selecting/deselecting points.} to apply the selection vector $\mathbf{x}$ to $\mathbf{P}$; and $\sigma$ is the cross-entropy loss created from the ground-truth label $\ell$ and the predicted label by $m$. Our problem definition is greatly inspired by \emph{attention} where neural networks focus on a meaningful subset of information~\cite{DBLP:journals/corr/VaswaniSPUJGKP17}. Although Eq.~\eqref{eq:proposed} is a minimization problem, we can still apply the adaptive gradient ascent by maximizing its negative value.

\noindent\textbf{Datasets.}
We use the Princeton ModelNet40~\cite{7298801} which contains 12,308 samples from 40 different object classes, e.g., sofa, desk, chair, etc. 2,468 samples were reserved for our resampling test and others were used to train PointNet.

\noindent\textbf{Baselines and Hyper-parameters.}
We tested four baselines: random resampling (Random~\cite{cloudcompare}), Octree~\cite{cloudcompare}, weighted locally optimal projection (WLOP~\cite{cgal:eb-19a}), and grid resampling (Grid~\cite{cgal:eb-19a}). As there are no commonly-accepted evaluation metrics for PCR, we compare our KP-based PCR with the baselines using the classification performance by PointNet.


We set $B=0.05$, $\xi=0.1$, $\lrate=0.0001$, $k=3,000$, and initialized $e_i$ with the LeCun normal initializer~\cite{Sutskever:2013:IIM:3042817.3043064} for our method. In particular, we tested with a low resampling ratio $B=0.05$, which is highly challenging. We use PTE to generate $x_i$ from $e_i$. For those baselines, we set the same resampling ratio of $0.05$. However, Octree and Grid sometimes resample more than 0.05 and their average resampling ratios are 0.06 as shown in Table~\ref{tbl:pcr}.

\noindent\textbf{Experimental Results.}
We feed each resampling result to PointNet and check if PointNet correctly recognizes its original class. In Table~\ref{tbl:pcr}, we report the accuracy of each resampling method. Our method shows the best accuracy, which proves that our method works well even with the complicated neural network-based objective. We also show some resampling results in Figure~\ref{fig:airplane}. In the figure, note that our result has the highest logit value. While this may be controversial, PointNet thinks our result is the best, which shows a potential discrepancy between machine and human perception --- recall that our purpose is not to resample points compliant with human visual perception but to show the proposed optimization technique works well. In Figure~\ref{fig:airplane}, Grid shows good visual quality, but its resampling ratio is 5.7\%.

\begin{table}[]
\footnotesize
\centering
\caption{PointNet accuracy after resampling points with various methods.}\label{tbl:pcr}
\begin{tabular}{|c|c|c|c|c|c|c|}
\hline
 & \multirow{2}{*}{\begin{tabular}[c]{@{}c@{}}Original\\(resampling ratio = $1.0$)\end{tabular}} & \multirow{2}{*}{\begin{tabular}[c]{@{}c@{}}Random\\($0.05$)\end{tabular}} & \multirow{2}{*}{\begin{tabular}[c]{@{}c@{}}Octree\\($0.06$)\end{tabular}} & \multirow{2}{*}{\begin{tabular}[c]{@{}c@{}}WLOP\\($0.05$)\end{tabular}} & \multirow{2}{*}{\begin{tabular}[c]{@{}c@{}}Grid\\($0.06$)\end{tabular}} & \multirow{2}{*}{\begin{tabular}[c]{@{}c@{}}Ours\\($0.05$)\end{tabular}} \\
 &  &  &  &  &  &  \\ \hline
PointNet Micro F-1 & 0.8842 & 0.6353 & 0.7208 & 0.7070 & 0.7414 & 0.8268\\ \hline
PointNet Macro F-1 & 0.8304 & 0.5398 & 0.6344 & 0.6150 & 0.6607 & 0.7457\\ \hline

\end{tabular}
\end{table}	


	\begin{figure}
	    \centering
	    \subfigure[Original]{\includegraphics[width=0.24\columnwidth]{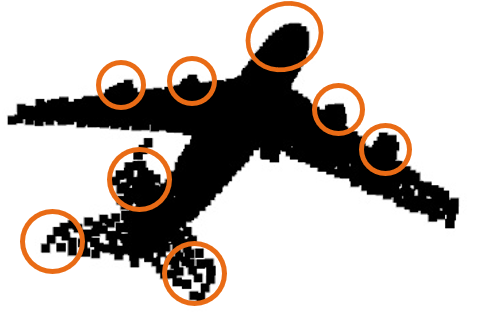}}
	    \subfigure[Ours (ratio=5\%,    logit=40.8)]{\includegraphics[width=0.24\columnwidth]{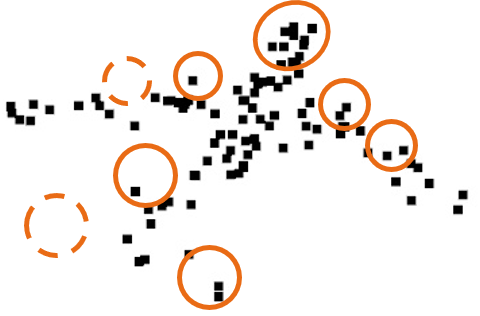}}
	    \subfigure[Grid (ratio=5.7\%,    logit=37.3)]{\includegraphics[width=0.24\columnwidth]{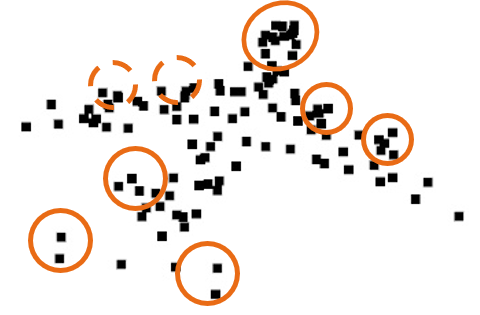}}
	    \subfigure[WLOP (ratio=5\%,    logit=31.4)]{\includegraphics[trim={0 0 0 2.2cm},clip,width=0.24\columnwidth]{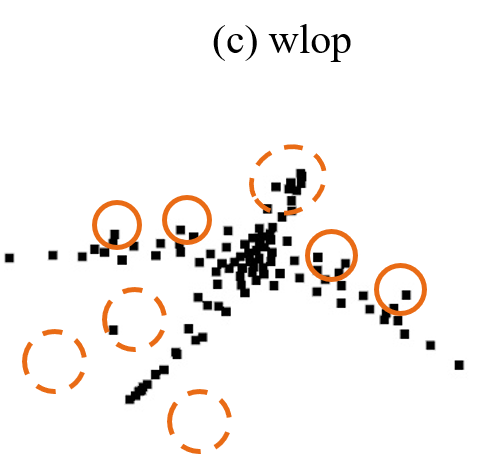}}
	    \caption{PCR Visualization. The logit of the airplane class is extracted from PointNet. We remove Random and Octree due to their relatively poor quality for this object.}\label{fig:airplane}
	\end{figure}

\section{Conclusion}
We proposed a novel method to solve large-scale 0-1 KPs using advanced deep learning techniques. We showed that our method is better at solving 0-1 KPs than other mathematical solvers. In our application to PCR, the proposed adaptive gradient ascent method optimized the neural network-based complicated objective well.

\end{document}